\documentclass[runningheads]{llncs}

\usepackage{graphicx}
\usepackage{amsmath}
\usepackage{amssymb}
\usepackage{amstext}
\usepackage{thm-restate}
\usepackage{bm}
\usepackage{float}
\usepackage[thinlines]{easytable}
\usepackage[numbers,sort&compress]{natbib}
\usepackage{tikz}
\usepackage{pgfplots}
\usepackage{dsfont}

\DeclareMathOperator{\bias}{Bias}
\DeclareMathOperator{\E}{\mathbb{E}}
\DeclareMathOperator{\D}{\mathcal{D}}

\newcommand*\dif{\mathop{}\!\mathrm{d}}
\newcommand{\LO}{L_{\text{overfit}}}

\DeclareMathOperator*{\argmin}{\arg\!\min}

\pgfplotsset{compat=1.16}

\usetikzlibrary{arrows,automata,positioning,calc,shapes.multipart,chains,babel,fit,backgrounds}

\tikzset{%
    block/.style={draw, fill=white, rectangle, 
            minimum height=2em, minimum width=2em},
    input/.style={inner sep=0pt},       
    output/.style={inner sep=0pt},      
    sum/.style = {draw, fill=white, circle, minimum size=2mm, node distance=1.5cm, inner sep=0pt},
    pinstyle/.style = {pin edge={to-,thin,black}}
}
\tikzstyle{block} = [draw, rectangle, text width=1.5cm, text centered, minimum height=1cm, node distance=3.3cm,fill=white]
\tikzstyle{container} = [draw, rectangle, inner sep=0.8cm, fill=white,minimum height=2cm]
\def\bottom#1#2{\hbox{\vbox to #1{\vfill\hbox{#2}}}}
\tikzset{
  mybackground/.style={execute at end picture={
      \begin{scope}[on background layer]
        \node[] at (current bounding box.south){\bottom{1cm} #1};
        \end{scope}
    }},
}

\begin{document}

\title{An Information-Theoretic Perspective on Overfitting and Underfitting}

\author{Daniel Bashir,
George D.\ Monta{\~n}ez,
Sonia Sehra,
Pedro Sandoval Segura,
Julius Lauw
}
\authorrunning{Bashir et al.}
\institute{{AMISTAD Lab\\Department of Computer Science, Harvey Mudd College, CA, USA}\\
\email{\{dbashir,gmontanez,ssehra,psandovalsegura,julauw\}@hmc.edu}
}

\maketitle  

\begin{abstract}
{
We present an information-theoretic framework for understanding overfitting and underfitting in machine learning and prove the formal undecidability of determining whether an arbitrary  classification algorithm will overfit a dataset.  Measuring algorithm capacity via the information transferred from datasets to models, we consider mismatches between algorithm capacities and datasets to provide a signature for when a model can overfit or underfit a dataset. We present results upper-bounding algorithm capacity, establish its relationship to quantities in the algorithmic search framework for machine learning, and relate our work to recent information-theoretic approaches to generalization. 
}
\keywords{overfitting, underfitting, algorithm capacity}
    
\end{abstract}

\section{\uppercase{Introduction}}
\label{sec:introduction}
\noindent
Overfitting and underfitting are constant and ubiquitous dangers in machine learning. The goal of supervised learning is to approximate or fit a true signal relating features $X$ to responses $Y$, which can be interpreted as a function $f:X\to Y$ or a distribution $P(Y|X)$. Overfitting occurs when an algorithm reduces error through memorization of training examples, with noisy or irrelevant features, rather than learning the true general relationship between $X$ and $Y$ \cite{zhang2016nn,krueger2017deep}. Underfitting occurs when an algorithm lacks sufficient model capacity or sufficient training to fully learn the true relationship, whether through memorization or not. A \emph{learning algorithm} (our general term for a machine learning approach that processes data to produce models or hypotheses) is equipped with a hypothesis space $\mathcal{G}$ which contains the potential guesses the algorithm may have for the target function (or distribution). The hypotheses in $\mathcal{G}$ can overfit when the complexity of the available hypotheses are mismatched to the complexity of the true signal being learned, allowing the act of model selection to become equivalent to the setting of bits in a general memory storage device, recording memorized label mappings as a short-cut for reducing training error.

While every machine learning practitioner is warned to avoid the twin pitfalls of overfitting and underfitting, theory related to both remains underdeveloped, largely relying on folk-wisdom and heuristic approaches. In particular, beyond intuition and comparative arguments, there is no formalized set of criteria for whether a particular algorithm will overfit or underfit a given dataset.

Of course, there exist well-known characterizations of algorithm complexity. The VC dimension provides a loose upper bound on model complexity in general, while more recent characterizations like Rademacher complexity seek to provide a normalized comparative measure. To measure the complexity of datasets as opposed to algorithms or models, Li and Abu-Mostafa's work provides a useful characterization \cite{li2006complexity}. Even given this existing work, recent papers (e.g., \cite{zhang2016nn}) have correctly pointed out that current theories of generalization are not enough to explain phenomena like the observed performance of deep neural networks. 

The search framework proposed in~\cite{montanez2017famine} abstracts the machine learning problem into a more general search setting. Using recent developments from this framework coupled with information-theoretic insights, we can gain a new perspective on issues of algorithm capacity, overfitting, and underfitting. Our definitions of algorithm capacity closely resemble recent information-theoretic analyses that consider the input-output mutual information of a learning algorithm, such as those in~\cite{xu-raginsky} and~\cite{pmlr-v83-bassily18a}, and bring to mind empirical \emph{generalized information} \cite{generalizedinfo}.

Our manuscript's primary contributions are as follows:
\begin{enumerate}
    \item We show that the general problem of determining whether an arbitrary algorithm will overfit a particular dataset is formally undecidable, by a reduction from the halting problem. This leads us to focus on special cases where overfitting determinations can be made, such as with restricted complexity learning algorithms~\cite{vapnik1999overview}.
    
    \item We develop a framework for explaining the phenomenon of overfitting and underfitting from an information-theoretic perspective, using notions of \emph{algorithm capacity} to redefine both pitfalls and demonstrate how those notions interact with existing work. 
\end{enumerate}

The remainder of the paper is organized as follows. In Section~\ref{sec:UNDECIDABILITY} we present the formal undecidability of the overfitting detection problem for classification algorithms under the standard observational definition of overfitting. Section~\ref{sec:DEFINITIONS} introduces our notions of capacity as well as new definitions for overfitting and underfitting in terms of capacity. We then present bounds on algorithm capacity, and conclude with a discussion of the significance of the results and future work. We begin by introducing some basic concepts and notational conventions, before looking at the formal undecidability of overfitting in the next section.

\subsection{Algorithms and Datasets}

A learning algorithm $\mathcal{A}$ may be viewed as a stochastic map $P_{G|D}$ that takes as input a training set $D$ of size $n$, namely $D = (Z_1,...,Z_n)$ whose elements belong to an instance space $\mathcal{Z}$ and are typically sampled i.i.d. according to an unknown distribution $\mathcal{D}$, and outputs a hypothesis in its hypothesis space, $g \in \mathcal{G}$. Each instance $Z_i$ in the dataset may also represent a pair $(\textbf{x}_i,y_i)$ where $\textbf{x}_i$ is a feature vector and $y_i$ is the corresponding label or response. $G$ denotes the random variable representing the output of $\mathcal{A}$ with $D$ as input.

Within this work, we limit ourselves to discrete hypothesis spaces and datasets. This allows us to use the discrete entropy throughout and reflects finite-precision numerical representations in all modern digital hardware. However, it also poses certain mathematical restrictions, such as excluding reproducing kernel Hilbert spaces and other continuous mathematical spaces from our consideration.

\section{\uppercase{Formal Undecidability of Overfitting}}\label{sec:UNDECIDABILITY}

Traditionally, overfitting is diagnosed by comparing the losses of an algorithm on training and tests datasets, where the error on the test set (average observed loss) is intended to approximate the true risk (expected loss). Observationally, if the true risk $R_{\mathcal{D}}(g)$ (estimated by sampling of test datasets) exceeds the empirical risk $\widehat{R}_D(g)$ (the risk on the training dataset), the algorithm is said to \textbf{overfit}.

\begin{definition}[Overfitting (Observational)]
\label{defn:obs_overfitting}
Algorithm $\mathcal{A}$ \textbf{overfits} dataset $D$ if it selects a hypothesis $g \in \mathcal{G}$ such that $R_{\mathcal{D}}(g) > \widehat{R}_D(g).$
\end{definition}
Under this definition, the problem of determining whether an arbitrary classification algorithm will overfit an arbitrary dataset is formally undecidable.

\begin{restatable}[Formal Undecidability of Overfitting]{theorem}{undecidabilityOverfitting}\label{thm:UNDECIDABLE-OVERFITTING}

Let $S$ be the set of all encodable classification algorithms, let $\langle \mathcal{A}\rangle$ denote the encoded form of algorithm $\mathcal{A}$, and let $D$ denote a dataset. Then, $$\LO = \{\langle \mathcal{A}\rangle, D | \mathcal{A} \in  S, \mathcal{A} \text{ will overfit } D\}$$ is undecidable.
\end{restatable}

\textit{Proof sketch}: We show $\LO$ is undecidable by a reduction from the halting problem. By way of contradiction, if $\LO$ is decidable then there exists a Turing machine, $M_{\text{overfit}}$, which for all inputs of the form $\langle \mathcal{A}\rangle$, $D$ halts and determines whether $\mathcal{A}$ will overfit $D$ either once training ends or asymptotically.

We construct a decider for $L_{\text{halt}}$: the machine \textit{$\mathcal{A}'$ builder} takes as input $\langle M \rangle, w$ and constructs an encoded algorithm $\mathcal{A}'$ which represents an iterative machine learning method, and exports the encoded algorithm along with a training dataset $D \sim \mathcal{D}$. On its first iteration, $\mathcal{A}'$ instantiates a machine learning model which produces maximally wrong (under the fixed loss function) response values for all examples in $D$, and uniformly randomly guesses response values for all examples not in $D$. Then $\mathcal{A}'$ will have lower expected error on any test set drawn from $\mathcal{D}$ than its observed (maximal) error on $D$, and will not overfit.

Next, $\mathcal{A}'$ simulates $M$ on $w$. If $M$ halts on $w$, $\mathcal{A}'$ memorizes dataset $D$ via a look-up table, still uniformly randomly guessing examples not in $D$. If $M$ does not halt on $w$, the original model remains. Thus, $\mathcal{A}'$ will overfit if and only if $M$ halts on $w$: asking $M_{\text{overfit}}$ whether $\mathcal{A}'$ will overfit on $D$ tells us if $M$ halts on $w$. $\qed$

Theorem~\ref{thm:UNDECIDABLE-OVERFITTING} tells us that there can exist no general purpose overfitting detector that can perfectly determine whether an arbitrary algorithm will overfit a dataset if trained to completion or for enough iterations (with respect to iterative methods). However, a less accurate detector can exist, such as a uniform random guesser. Given that accuracy on this problem ranges between $0$ and $100$,  exclusive, the question becomes: how much can we improve an overfitting detector to get accuracy closer to (though never quite reaching) 100\%? The proof for Theorem~\ref{thm:UNDECIDABLE-OVERFITTING} hints at what property can be leveraged to improve such detectors. In the proof, we rely on the fact that the algorithm can memorize a dataset of arbitrary complexity, essentially giving it unlimited algorithmic capacity. Put another way, we assumed a model family with infinite VC dimension, able to discriminate any set of points, no matter how large. Since, in classification settings, restricted VC dimension is both necessary and sufficient for uniform convergence of empirical error to generalization error, this suggests that comparing algorithm capacity to dataset complexity could play a role in improving such detection methods.

\section{\uppercase{Definitions and Terminology}}
\label{sec:DEFINITIONS}
\subsection{Capacity, Overfitting, and Underfitting}
We present a set of definitions and theorems that characterize our view of learning algorithms, inspired in part by the \emph{algorithmic search framework} \cite{montanez2017famine}.

In the search framework, a search problem is specified by a tuple $(\mathrm{\Omega},T,F)$ consisting of a search space $\mathrm{\Omega}$, a target subset $T \subseteq \mathrm{\Omega}$, and an external information resource $F$. A search algorithm $\mathcal{A}$ at time $i$ in its execution computes a probability distribution $P_i$ over $\mathrm{\Omega}$ and samples an element $\omega_i \in \mathrm{\Omega}$ according to $P_i$, resulting in a sequence of distributions $\tilde{P} = [P_1, P_2,\ldots,P_N]$. $\mathcal{A}$ maintains a history $h$, where each timestep $h_i = (\omega_i,F(\omega_i))$ contains the element of $\mathrm{\Omega}$ that $\mathcal{A}$ sampled at time $i$ and the information about $\omega_i$ provided by $F$. A search is considered successful if, given a fixed target set $t$, $\mathcal{A}$ samples an element $\omega_i \in t$, where $t$ is represented by a $|\mathrm{\Omega}|$-length binary target vector $\mathbf{t}$, which has a $1$ at index $i$ if and only if $\omega_i \in t$, namely, $\mathbf{t}_i = \mathds{1}_{\omega_i \in t}$ where $\mathds{1}_{\_}$ denotes the indicator function. The quantity $q(t,F)$ denotes the expected per-query probability of success, or the expected probability mass placed on $t$.

For our purposes, the search space $\mathrm{\Omega}$ is the hypothesis space $\mathcal{G}$ of $\mathcal{A}$. The information resource has two components. The initial information $F(\emptyset)$ is a training dataset $D$ of size $n$, while $F(h)$ is a non-negative loss function $\ell : \ \mathcal{G} \times \mathcal{Z} \rightarrow \mathbb{R}_{\geq 0}$. The target set $T$ consists of all $g \in \mathcal{G}$ that achieve low population risk, namely, $$R_{\mathcal{D}}(g) = \mathbb{E}_{\D}[\ell(g,\mathcal{Z})] = \int_{\mathcal{Z}} \ell(g,z)\mathcal{D}(\dif z) < \epsilon$$
for some fixed scalar $\epsilon > 0$ for any data-generating distribution $\mathcal{D}$. The history of sampled elements corresponds to hypotheses considered as $\mathcal{A}$ is trained, using a method such as stochastic gradient descent. 

Since $\mathcal{D}$ is unknown, we can instead compute the empirical risk of $g$ on dataset $D$ as $$\widehat{R}_D(g) := \frac{1}{n}\sum_{i=1}^n \ell(g,z_i)$$
in the search of a hypothesis $g^* = \argmin_{g\in \mathcal{G}} R_{\mathcal{D}}(g)$, for which we choose as a proxy the empirical risk minimizer (ERM hypothesis), $\hat{g} = \argmin_{g\in \mathcal{G}} \widehat{R}_D(g)$.

We will frequently use the term ``capacity'' to describe the learning capabilities of an algorithm, as opposed to algorithm complexity, which indicates the expressiveness of functions in the algorithm's hypothesis space $\mathcal{G}$ (e.g., linear functions for a regression model).

\begin{definition}[Algorithm Capacity]
\label{defn:capacity}
The \textbf{capacity} $C_{\mathcal{A}}$ of an algorithm $\mathcal{A}$ is the maximum amount of information that $\mathcal{A}$ can extract from a dataset $D \sim \mathcal{D}$ when selecting its output hypothesis $g$, namely, $$C_{\mathcal{A}} = \underset{\mathcal{D}}{\sup}\;  I(G;D)$$
where $G$ takes values in $\mathcal{G}$. 
\end{definition}
$P(G|D)$ is fixed by the algorithm $\mathcal{A}$. Therefore, our definition for capacity is equivalent to the input-output mutual information measure given in~\cite{xu-raginsky} and can be viewed as the maximum capacity of an information channel from $\mathcal{Z}^n$ to $\mathcal{G}$. Note that the maximum amount of information an algorithm $\mathcal{A}$ may transfer from a dataset in selecting a hypothesis is the number of bits required to memorize a one-to-one mapping between each feature-label pair in that dataset.

For a fixed distribution $\D$, we can define the capacity relative to that particular distribution, which is simply the mutual information.
\begin{definition}[Distributional Algorithm Capacity]
\label{defn:distributional-capacity}
For $D\sim \D$,
$$C_{\mathcal{A},\D} = I(G;D).$$
\end{definition}

In an algorithm's search for an ERM hypothesis by an iterative method such as gradient descent, we may regard each iteration as a timestep and observe that by a time $i$, $\mathcal{A}$ will have sampled only a subset of $\mathcal{G}$. This may reduce the entropy of a variable $G_i$ drawn from the expected $i$th distribution, $\E[P_i|D]$, and motivates the following definition.

\begin{definition}[Time-indexed Capacity]
\label{defn:time}
Let $P_i$ denote the (stochastic) probability distribution over $\mathcal{G}$ at time $i$. $\mathcal{A}$'s capacity \textit{at time i} is the maximum amount of information $\mathcal{A}$ may transfer from $D \sim \D$
to $G_i \sim \E[P_i|D]$,
\begin{align*}
    C^i_{\mathcal{A}} &= \sup_{\D} I(G_i;D).
\end{align*}
\end{definition}

Finally, we define the pointwise information transfer by an algorithm from a specific dataset to a specific hypothesis.
\begin{definition}[Pointwise Information Transfer]
\label{defn:pointwise-transfer}
For a given dataset $d$ and specific hypothesis $g$, the \textbf{pointwise information transfer} by algorithm $\mathcal{A}$ from $d$ to $g$ is the pointwise mutual information (lift),
\begin{align*}
    C_{\mathcal{A}}(g,d) &= \log_2 \frac{p(g,d)}{p(g)p(d)} = \log_2 \frac{p(g|d)}{p(g)} = \log_2 \frac{p(d|g)}{p(d)}. 
\end{align*}
\end{definition}
Note that $p(g|d)$ captures how representative a hypothesis is of a dataset (i.e., how deterministic is the algorithm?), while $p(d|g)$ measures how identifiable a dataset is given a hypothesis (i.e., how many datasets strongly map to $g$?). For deterministic algorithms, $C_{\mathcal{A}}(g,d)$ becomes the Shannon surprisal of the set of datasets producing $g$, $C_{\mathcal{A}}(g,d) = -\log_2 \D(S)$, where $S =\{d' \in \mathcal{Z}^n \mid \mathcal{A}(d') = g\}$ is the collection of datasets $d'$ such that $p(g|d') = 1$ under $\mathcal{A}$. Taking the expectation with respect to $G$ and $D$, we see that $\E_{G,\D}[C_{\mathcal{A}}(G,D)] = C_{\mathcal{A},\D}$.

Having provided definitions of algorithm capacity, we next consider dataset complexity. Comparing an algorithm's capacity to the complexity of the dataset it is trained on may give insight into whether the algorithm will overfit or underfit. We begin with a definition based on algorithmic compressibility from~\cite{li2006complexity}.

\begin{definition}[Dataset Turing Complexity]
Given a fixed Turing machine $M$ that accepts a string $p$ and feature vector $\mathbf{x}$ as input and outputs a label $y$, the \textit{data complexity} of a dataset $D$ is $$C_{D,M} = L(\langle M \rangle) + L(p)$$
where $L(p) = \min \{|p|: \forall (\mathbf{x},y) \in D, M(p,\mathbf{x}) = y\}$. That is, the data complexity $C_{D,M}$ is the length of the shortest program that correctly maps every input in the data set $\mathcal{D}$ to its corresponding output. 
\end{definition}
For a dataset $D = (Z_1,...,Z_n)$, the above definition contrasts with $C'_D = \sum_{i=1}^nb(z_i)$, where $b(z_i)$ is the number of bits required to encode the feature-label pair $z_i$ without compression. $C'_D$ gives the number of bits required to memorize an arbitrary dataset $D$. Taking the minimum of these two defined quantities gives us our definition of dataset complexity.
\begin{definition}[Dataset Complexity $C_D$]
$C_D = \min \{C_{D,M}, C'_D\}$.
\end{definition}
By construction $C_D \leq C'_D$, giving us a computable upper bound on dataset complexity. While $C_{D,M}$ is not explicitly computable, methods for estimating the quantity are proposed in~\cite{li2006complexity}, to which we refer the interested reader. Given the definitions of dataset complexity and algorithm capacity, we can now define overfitting in information-theoretic terms.
\begin{definition}[Overfitting]
\label{defn:overfitting}
An algorithm $\mathcal{A}$ \textbf{overfits} if $$C_{\mathcal{A},\D} > \E_{\mathcal{D}}[C_D],$$
i.e., the algorithm tends to extract more bits than necessary to capture the noiseless signal from the dataset. The degree of overfitting is given by $C_{\mathcal{A},\D} - \E_{\mathcal{D}}[C_D]$.
\end{definition}

Like overfitting, we can also give an information-theoretic definition for underfitting, based on the time-indexed capacity from Definition~\ref{defn:time}.
\begin{definition}[Underfitting]
An algorithm $\mathcal{A}$ \textbf{underfits}  at iteration $i$ if $$C^i_{\mathcal{A}} < \E_{\mathcal{D}}[C_D]$$
i.e., after training for $i$ timesteps, $\mathcal{A}$ has capacity strictly less than $\E_{\mathcal{D}}[C_D]$.
\end{definition}

If the algorithm's model does not contain enough information to accomplish a learning task, this could be the result of insufficient capacity, insufficient training, or insufficient information retention, all of which are captured by $C^i_{\mathcal{A}}$.

Lastly, by Definition~\ref{defn:pointwise-transfer} we can define the overfitting between a fixed hypothesis (model) $g$ and a fixed dataset $d$, related to the \emph{generalized information} of Bartlett and Holloway \cite{generalizedinfo} when considering $C_{\mathcal{A}}(g,d) - C_d$.
\begin{definition}[Model Overfit]
$\mathcal{A}$'s \textbf{model $g$ overfits $d$} if $C_{\mathcal{A}}(g,d) > C_d$.
\end{definition}
Because $C_{D,M}$ is uncomputable, one cannot generally determine model overfit whenever $C_{\mathcal{A}}(g,d) \leq C'_d$ (in agreement with Theorem~\ref{thm:UNDECIDABLE-OVERFITTING}). However, one can claim model overfitting for the special case of $C_{\mathcal{A}}(g,d) > C'_d \geq C_d$.

\subsection{Capacity, Bias, and Expressivity}

We now review quantitative notions of bias and expressivity introduced in~\cite{lauw2020bias}. Just as the estimation bias of a learning algorithm trades off with its variance, the algorithmic bias also trades off with expressivity, which loosely captures how widely a learning algorithm distributes probability mass over its search space in expectation. Naturally, this will be affected by how well an algorithm's inductive bias aligns with the target vector. To that end, we introduce the inductive orientation of an algorithm.

\begin{definition}[Inductive Orientation]
Let $F$ be an external information resource, such as a dataset, $\tilde{P}$ be defined as above,  $H$ be an algorithm's search history, and let $$\overline{\mathbf{P}}_{F} := \mathbb{E}_{\tilde{P}, H} \left[ \frac{1}{|\tilde{P}|} \sum_{i=1}^{|\tilde{P}|} \mathbf{P}_i \bigg| F \right].$$ That is, $\overline{\mathbf{P}}_{F}$ is the algorithm's expected average conditional distribution on the search space given $F$. Then the \textbf{inductive orientation} of an algorithm is 
\begin{align}
    \overline{\mathbf{P}}_{\D}&= \E_{F\sim\D}[\overline{\mathbf{P}}_{F}]
\end{align}
\end{definition}

We may now define the entropic expressivity of an algorithm.

\begin{definition}[Entropic Expressivity]
The \textbf{entropic expressivity} of an algorithm is the Shannon entropy of its inductive orientation,
\begin{align*}
        H(\overline{\mathbf{P}}_{\D}) 
            &= H(\mathcal{U}) - D_{\mathrm{KL}}(\overline{\mathbf{P}}_{\D} \;||\; \mathcal{U})
    \end{align*}
where $D_{\mathrm{KL}}(\overline{\mathbf{P}}_{\D} \;||\; \mathcal{U})$ is the Kullback-Leibler divergence between distribution $\overline{\mathbf{P}}_{\D}$ and the uniform distribution $\mathcal{U}$, and both are distributions over the search space $\mathrm{\Omega}$.
\end{definition}

Lauw et al.\  demonstrate a quantitative trade-off between the entropic expressivity and the bias of a learning algorithm~\cite{lauw2020bias}. This trade-off will allow us to relate algorithm capacity to bias, as well. As our hypothesis space $\mathcal{G}$ is the relevant search space here, we substitute $\mathcal{G}$ for $\mathrm{\Omega}$
 throughout.
 
\begin{definition}
\label{def:bias_D}
    (Algorithmic Bias) Given a fixed target function \(\mathbf{t}\) corresponding to the target set $t$, let $p = \|\mathbf{t}\|^2/|\mathcal{G}|$ denote the expected per-query probability of success under uniform random sampling, $\mathbf{P}_{\mathcal{U}} = \bm{1}\cdot |\mathcal{G}|^{-1}$ be the inductive orientation vector for a uniform random sampler, and $F \sim \mathcal{D}$, where $\D$ is a distribution over a collection of information resources $\mathcal{F}$. Then, 
    \begin{align*}
        \bias(\mathcal{D}, \mathbf{t})
        &= \E_{\mathcal{D}}[q(t,F) - p]\\
        &= \mathbf{t}^{\top}( \overline{\mathbf{P}}_{\D} - \mathbf{P}_{\mathcal{U}})\\
        &= \mathbf{t}^\top\mathbb{E}_{\mathcal{D}} \left[ \overline{\mathbf{P}}_{F}\right] - \mathbf{t}^{\top} (\bm{1}\cdot |\mathcal{G}|^{-1}) \\
        &= \mathbf{t}^\top  \int_{\mathcal{F}} \overline{\mathbf{P}}_{f} \mathcal{D}(f) \dif f - \frac{\|\mathbf{t}\|^2}{|\mathcal{G}|}.
    \end{align*}
\end{definition}

Following~\cite{lauw2020bias}, we re-state the bias-expressivity trade-off:

\begin{restatable}[Bias-Expressivity Trade-off]{theorem}{tradeoff}
    \label{thm:tradeoff}
    Given a distribution over information resources $\D$ and a fixed target $t \subseteq \mathcal{G}$, entropic expressivity is bounded above in terms of bias,
    $$H(\overline{\mathbf{P}}_{\D}) \leq \log_2 |\mathcal{G}| - 2 \bias(\D, \mathbf{t})^2.$$ 
    Additionally, bias is bounded above in terms of entropic expressivity,
    \begin{align*}
        \bias(\D,\mathbf{t}) 
            &\leq \sqrt{\frac{1}{2}(\log_2|\mathcal{G}| - H(\overline{\mathbf{P}}_{\D}))} \\
            &= \sqrt{\frac{1}{2} D_{\text{KL}}(\overline{\mathbf{P}}_{\D} \;||\; \mathcal{U})}.
    \end{align*} 
\end{restatable}

Given the notions of inductive orientation and entropic expressivity,  we can define distributional algorithm capacity in terms of these quantities.

\begin{restatable}[Distributional Capacity as Entropic Expressivity]{theorem}{distributionalCapacityExpressivity}\label{thm:DISTRIBUTIONAL-CAPACITY-EXPRESSIVITY}

An algorithm's distributional capacity may be re-written as the difference between its entropic expressivity and its expected entropic expressivity, namely
\begin{align}
    C_{\mathcal{A},\D} &= H(\overline{\mathbf{P}}_{\mathcal{D}}) - \mathbb{E}_{\mathcal{D}}[H(\overline{\mathbf{P}}_{F})].
\end{align}
\end{restatable}

\begin{proof}
Note that, by marginalization of $F \sim \mathcal{D}$ and the definition of $\overline{P}_{\mathcal{D}}$,
$$p(g) = 
\E_{\mathcal{D}}[p(g|F)] = \E_{\mathcal{D}}[\overline{P}_F(g)] = \overline{P}_{\mathcal{D}}(g).$$
 Therefore,
\begin{align*}
    H(G) &= -\sum_{g \in G}p(g) \log p(g) = -\sum_{g \in G} \overline{P}_{\mathcal{D}}(g) \log \overline{P}_{\mathcal{D}}(g) = H(\overline{\mathbf{P}}_{\mathcal{D}}).
\end{align*}
Furthermore,
\begin{align*}
    H(G|F) &=  \sum_f H(G|F=f) P(f) \\ 
    &=  -\sum_f \sum_g  P(f) [p(g|F=f) \log p(g|F = f)]
\end{align*}
\begin{align*}
&=  -\sum_f P(f) \left[\sum_g  p(g|F=f) \log p(g|F = f)\right] \\
&= E_{\mathcal{D}}\left[ -\sum_g  p(g|F) \log p(g|F)\right] \\
&= E_{\mathcal{D}}\left[-\sum_g  \overline{P}_{F}(g) \log \overline{P}_{F}(g)\right] \\
&= E_{\mathcal{D}}[H(\overline{\mathbf{P}}_F)].
\end{align*}

Then, by the definition of distributional capacity $C_{\mathcal{A}, \mathcal{D}}$ for $F\sim \D$,
\begin{align*}
    C_{\mathcal{A}, \mathcal{D}} &= I(G;F) = H(G) - H(G|F)  = H(\overline{\mathbf{P}}_{\mathcal{D}}) - \mathbb{E}_{\mathcal{D}}[H(\overline{\mathbf{P}}_{F})]. \hspace{1em} \qed
\end{align*}
\end{proof}
\textbf{Note:} Theorem~\ref{thm:DISTRIBUTIONAL-CAPACITY-EXPRESSIVITY} considers a distribution vector $\overline{\mathbf{P}}_{\D}$ that is averaged over all iterations of a search; if the distribution averaged over only the final iteration is desired, $\overline{\mathbf{P}}_{n,\D}$ and $\overline{\mathbf{P}}_{n,F}$ can be used instead, as detailed in \cite{sam2020DPSM}.

Theorem \ref{thm:DISTRIBUTIONAL-CAPACITY-EXPRESSIVITY} points towards a way of empirically estimating the quantity $C_{\mathcal{A},\D}$. We first form a \textit{labeling distribution matrix} (LDM) for algorithm $\mathcal{A}$ described in~\cite{pss-2020-LDM}: the matrix consists of $K$ simplex vectors $P_{f_1},...,P_{f_k}$, where simplex vector $P_{f_i}$ corresponds to the probability distribution that $\mathcal{A}$ induces over its search space $\mathrm{\Omega}$ after being trained on information resource $f_i$ drawn from $\mathcal{D}$. Taking the average of all columns converges toward $\overline{\mathbf{P}}_{\mathcal{D}}$ by the law of large numbers (with increasing $K$), and taking the entropy of the averaged column vector converges toward $H(\overline{\mathbf{P}}_{\mathcal{D}})$. Furthermore, averaging the entropies of each column vector in the matrix will converge toward $\mathbb{E}_{\mathcal{D}}[H(\overline{\mathbf{P}}_{F})]$ as $K$ increases.

\begin{restatable}[Algorithm Capacity as Entropic Expressivity]{corollary}{capacityExpressivity}\label{cor:CAPACITY-EXPRESSIVITY}
\begin{align}
    C_{\mathcal{A}} &= \underset{ \mathcal{D}}{\mathrm{sup}}[ H(\overline{\mathbf{P}}_{\mathcal{D}}) - \mathbb{E}_{\mathcal{D}}[H(\overline{\mathbf{P}}_{F})]]
\end{align}
\end{restatable}

Expressing algorithm capacity in terms of entropic expressivity provides additional intuition about what precisely is being measured: Theorem \ref{thm:DISTRIBUTIONAL-CAPACITY-EXPRESSIVITY} illustrates that an algorithm's capacity may be interpreted as how much its entropic expressivity for a fixed distribution differs from its expected entropic expressivity. In other words, $H(\overline{\mathbf{P}}_{\mathcal{D}})$ captures how ``flat'' the expected induced probability distribution is, which could result either from averaging together flat distributions or by averaging together many ``sharp'' distributions that happen to place their mass on very different regions of the search space. In contrast, $\mathbb{E}_{\mathcal{D}}[H(\overline{\mathbf{P}}_{F})]$ measures how flat the induced distributions are in expectation. By subtracting the flatness aspect from the combined quantity that captures both flatness and dispersal of probability mass, we get a quantity that represents how much an algorithm shifts its probability mass in response to different information resources. The ability to do this is equivalent to the ability to store information (by taking on different configurations), and thus is a fitting measure of algorithm capacity. 

Furthermore, we can use the values of entropic expressivity to derive bounds on $C_{\mathcal{A},\mathcal{D}}$ based on the entropic expressivity bounds established in~\cite{lauw2020bias}. We explore these connections next.

\ifx

\begin{definition}[Dataset Complexity (adapted from \cite{li2006complexity})]
Given a fixed Turing Machine $M$ that accepts a string $p$ and feature vector $x$ as input and outputs a label $y$, the \textit{data complexity} of a dataset $D$ is $$C(D) = L(\langle M \rangle) + L(p)$$
where $L(p) = \min \{|p|: \forall (\textbf{x},y) \in D, M(p,\textbf{x}) = y\}$. That is, the data complexity $C_{\mathcal{U}}(D)$ is the length of the shortest program that can correctly map every input in the data set $\mathcal{D}$ to its corresponding output. 
\end{definition}

\begin{definition}[Saturation]
An algorithm $\mathcal{A}$ becomes \textbf{saturated} when $$C_D >> C_{\mathcal{A}}$$
due to the dataset's volume. That is, the volume of data is so large that the algorithm must exchange memorized bits of information for information concerning the true signal to reduce error.
\end{definition}

\begin{definition}[Excess Capacity]
An algorithm $\mathcal{A}$'s \textbf{excess capacity} with respect to a dataset $D$ is $$C_{excess} = C_{\mathcal{A}}-C_D$$ i.e. the capacity $\mathcal{A}$ has beyond what is needed to generalize from the given data.
\end{definition}
\fi

\section{\uppercase{Algorithm Capacity Bounds}}
\label{sec:ALG-CAPACITY-BOUNDS}
Our definition of algorithm capacity provides one concrete way of measuring what has traditionally been a loosely defined quantity. The goal of this section is to provide further insight by bounding algorithm capacity. For more detail on experimental methods to estimate algorithm capacity, we refer the reader to \cite{pss-2020-LDM}.

In addition, if we expand the possible hypothesis spaces under consideration to be real-valued functions, we obtain an upper bound on $C_{\mathcal{A}}$ in terms of the VC dimension as demonstrated in Section~\ref{sec:VCDIM-UPPER-BOUND}.

\subsection{Trade-off Bounds}
\label{sec:TRADE-OFF-BOUNDS}
First, we demonstrate how the bias-expressivity trade-off furnishes immediate bounds on algorithm capacity. Theorems \ref{thm:tradeoff} and \ref{thm:DISTRIBUTIONAL-CAPACITY-EXPRESSIVITY} give us our first capacity bound.

\begin{restatable}[Distributional Capacity Upper Bound]{theorem}{distributionalCapacityBound}
    \label{thm:dist-upper}
    \begin{align}
        C_{\mathcal{A},\mathcal{D}} \leq \log_2 |\mathcal{G}| - 2 \bias(\mathcal{D}, \mathbf{t})^2 - \mathbb{E}_{\mathcal{D}}[H(\overline{\mathbf{P}}_{F})].
    \end{align}
\end{restatable}

Using the range bounds from Theorem 5.3 in~\cite{lauw2020bias}, we can obtain even tighter bounds on $C_{\mathcal{A},\D}$ as a function of the bias, shown in Table~\ref{tab:capacityBounds}.
\begin{table}[H] 
    \centering
    \caption{Because the range of entropic expressivity changes with different levels of bias relative to target function $\mathbf{t}$, the maximum value for $C_{\mathcal{A},\D}$ does also.}
    \vspace{1em}
    \label{tab:capacityBounds}
    \begin{TAB}(r,0.3cm,1cm)[2pt]{|c|c|c|}{|c|c|c|c|}
        $\bias(\mathcal{D}, \mathbf{t})$ & $\mathbf{t}^{\top} \overline{\mathbf{P}}_{\D}$ & \textbf{Capacity Upper Bound}\\
        \begin{tabular}{@{}c@{}} $-p$ \\ (Minimum bias) \end{tabular} & $0$ &  $\log_2( |\mathcal{G}| - \|\mathbf{t}\|^2) - \mathbb{E}_{\mathcal{D}}[H(\overline{\mathbf{P}}_{F})]$ \\
        \begin{tabular}{@{}c@{}} $0$ \\ (No bias) \end{tabular} & $p$ & $\log_2 |\mathcal{G}| - \mathbb{E}_{\mathcal{D}}[H(\overline{\mathbf{P}}_{F})]$ \\
        \begin{tabular}{@{}c@{}} $1-p$ \\ (Maximum bias) \end{tabular} & $1$ & $\log_2 \|\mathbf{t}\|^2 - \mathbb{E}_{\mathcal{D}}[H(\overline{\mathbf{P}}_{F})]$ \\
    \end{TAB}
\end{table}

Furthermore, rewriting the mutual information as KL-divergence furnishes another bound on $C_{\mathcal{A},\mathcal{D}}$.

\begin{restatable}[Distributional Capacity KL-Divergence Bound]{theorem}{distributionalCapacityKL}
    \label{thm:dist-kl}
    \begin{align*}
        C_{\mathcal{A},\mathcal{D}} &= I(G;D) \\ &= \mathbb{E}_{\mathcal{D}}[D_{KL}(p_{G|D} || p_{G})] \\ &\leq \underset{\mathcal{D}}{\mathrm{sup}}[D_{KL}(p_{G|D} || p_{G})].
    \end{align*}
\end{restatable}

That is, the maximum information that can be transferred between a learning algorithm and a particular dataset is bounded above by the maximum divergence between a prior distribution over the hypothesis space and a posterior distribution over the hypothesis space, given a dataset.

\subsection{A VC Upper Bound}\label{sec:VCDIM-UPPER-BOUND}
The discussion in \cite{vapnik71uniform} allows us to recover an upper bound on algorithm capacity: the logarithm of the well-known VC dimension, which provides a value in bits.

\begin{restatable}[VC Dimension as Information Complexity]{theorem}{vcInfo}\label{thm:vc-info}
Suppose algorithm $\mathcal{A}$ utilizes a hypothesis space of real-valued functions $\mathcal{G}$. Then
$$C_{\mathcal{A}} \leq \log d_{VC}(\mathcal{G}).$$
\end{restatable}

\section{\uppercase{CONCLUSION}}
Confronted with the ever-present dangers of overfitting and underfitting, we develop an information-theoretic perspective for understanding these phenomena, allowing us to characterize when they can occur and to what degree. We do so by considering the capacities of algorithms, complexities of datasets, and their relationship. In particular, we characterize overfitting as a symptom of mismatch between an algorithm's informational capacity and the complexity of the relationship it is attempting to learn. In colloquial terms, we have met the enemy, and it is mismatched capacity.

After introducing variations on algorithm capacity and recasting overfitting and underfitting as the relationship between algorithm capacity and dataset complexity, we give bounds on the algorithm capacity. We demonstrate that while the problem of determining whether an arbritrary classification algorithm will overfit a given dataset is formally undecidable, we can estimate the quantities proposed in this paper to determine when algorithm will overfit in expectation, and in some cases, when an algorithm's model overfits a dataset. Algorithm capacity estimation is the subject of future work.

Our methods make use of existing machinery from other frameworks, such as the algorithmic search framework and VC theory, which provide helpful characterizations and bounds for our current investigation. In particular, showing that distributional algorithm capacity can be written as a function of entropic expressivity allows us to gain insight into what algorithmic capacity means geometrically, in terms of shifted probability mass. In the future, the information-theoretic groundwork laid here will allow us to incorporate and extend other existing work, such as establishing a direct connection between the notions of bias, expressivity, and generalization, proving generalization bounds under our definitions and through our bounds on algorithm capacity.

\bibliographystyle{splncs04nat}
\bibliography{bibliography}

\section*{\uppercase{APPENDIX}: Supplementary Proofs}


\undecidabilityOverfitting*
\begin{proof}

\begin{figure*}[htbp]
  \centering
  \resizebox{.65\textwidth}{!}{
\begin{tikzpicture}[mybackground={M$_{\text{halt}}$}]

    \node [input, name=text1] {$\langle M \rangle, \omega$};
    \node [block, right=1.0cm of text1] (text2) {$\mathcal{A}'$\\Builder};
    \node [block, right of=text2] (text3) {$M_{\text{overfit}}$};
    
    \node [above right=0cm and 0.9cm of text3, node distance=5cm] (text4) {\color{green} Yes};
    \node [below of=text4, node distance=1.5cm] (text5) {\color{red} No};

    \begin{scope}[on background layer]
    \node [container,fit=(text2) (text3)] (container) {};
\end{scope}
    \draw [->] (text1) -- (text2);
    \draw [->] (text2) -- node [text width=2cm,midway,above,align=center ] {$\langle\mathcal{A}'\rangle, D$} (text3);
    \draw [->] (text3) -- node [above] {Yes} (text4);
    \draw [->] (text3) -- node [below] {No} (text5);
\end{tikzpicture}
  }
  \caption{$M_{\text{halt}}$ constructed using $M_{\text{overfit}}$.}
\end{figure*}
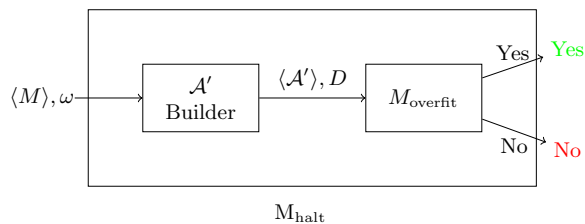

We show that $\LO$ is undecidable by a reduction from the halting problem. Assume, by way of contradiction, that $\LO$ is decidable. Then there exists a Turing machine, $M_{\text{overfit}}$, which for all inputs of the form $\langle \mathcal{A}\rangle$, $D$ halts and determines whether $\mathcal{A}$ will overfit $D$, either once training ends or asymptotically. 

    We construct a decider for $L_{\text{halt}}$ as follows. First, we create another machine called \textit{$\mathcal{A}'$ builder}, which takes as input $\langle M \rangle, w$, which are a Turing machine encoding and input string, respectively. $\mathcal{A}'$ builder constructs an encoded algorithm $\mathcal{A}'$ which represents an iterative machine learning method, and exports the encoded algorithm along with a training dataset $D$, consisting of some finite number of input features and output responses drawn from a generating distribution $\mathcal{D}$, where the output responses are drawn from a finite set of class labels. $\mathcal{A}'$ works in the following way. On its first iteration (or epoch), it instantiates a machine learning model which produces maximally wrong (under the fixed loss function) response values for all training examples in $D$, and uniformly randomly guesses response values for all examples not in $D$. Thus, as long as the probability of $D$ is not 1 under $\mathcal{D}$ (i.e., the distribution can produce some dataset other than $D$), $\mathcal{A}'$ will have lower expected error on any test set from $\mathcal{D}$ than its observed error on $D$. Therefore, the initial model does not overfit under Definition~\ref{defn:obs_overfitting}, since it will have higher prediction error on training data than on test data in expectation.
    
    Next, $\mathcal{A}'$ simulates $M$ on $w$. If $M$ halts on $w$, $\mathcal{A}'$ updates its internal model to memorize dataset $D$ via a simple look-up table, while still uniformly randomly guessing for all examples not in the training dataset. In this case, all training error is eliminated, and testing error will be higher in expectation than training error, due to memorization. If $M$ does not halt on $w$, the original non-overfitting model remains, and the algorithm only ever completes its single, initial iteration. Thus, $\mathcal{A}'$ will overfit if and only if $M$ halts on $w$. 
    
    Now, assuming that $M_{\text{overfit}}$ exists, we pass the outputs of $\mathcal{A}'$ builder to it, and ask if $\mathcal{A}'$ will eventually overfit on $D$. The answer will immediately tell us if $M$ halts on $w$, since it will eventually overfit if and only if $M$ halts on $w$. The outputs from $M_{\text{overfit}}$ are assigned to the output of $M_{\text{halt}}$, giving us a decider for $L_{\text{halt}}$, a contradiction. Thus, contrary to our initial assumption $M_{\text{overfit}}$ cannot exist, and $\LO$ is undecidable.
\end{proof}


\vcInfo*
\begin{proof}

We follow the conventions in~\cite{vapnik71uniform}.

Let $X$ be a set of elementary events on which a probability measure $P_X$ is defined. Let $S$ be a collection of events (subsets of $X$) which are measurable with respect to $P_X$. Let $X_r = x_1,...,x_r$ be a finite sample of elements of $X$. Each set $A$ in $S$ determines in $X_r$ a subsample $X_r^A = x_{i_1},...,x_{i_k}$ consisting of terms in the sample $X_r$ which belong to $A$.

The number of different subsamples of $X_r$ induced by sets in $S$ is defined as the \textit{index} of the system $S$ with respect to the sample $x_1,..,x_r$ and denoted $\Delta^S(x_1,...,x_r)$, which is upper bounded by $2^r$. The maximum of these indices over all samples of size $r$ is the \textit{growth function}: $$m^S(r) = \max \Delta^S(x_1,...,x_r).$$
In the learning setting, $S$ denotes the hypothesis space $\mathcal{G}$ of algorithm $\mathcal{A}$. $X_r$ is our dataset. Then the growth function is exactly the VC dimension $d_{VC}$. 

Finally, Vapnik and Chervonenkis define the entropy of the system of events $S$ in samples of size $l$ as $H^S(l)$. In their proof of Lemma 4, it is noted that if $l = nl_0$ (i.e. we partitioned $l$ into $n$ even parts), the expectation of the logarithms of indexes is as follows: $$\mathbb{E} \left[ \frac{1}{n} \sum_{i=0}^{n-1} \log_2 \Delta^S(x_{il_0+1},...,x_{(i+1)l_0}) \right] = H^S(l_0).$$

The collection of events $S$ maps to the hypothesis space $\mathcal{G}$, while our dataset sample is $X_r$. Further, the algorithm's entropy, i.e., the capacity, is the quantity $H^S(l_0)$. As pointed out earlier, $d_{VC}(\mathcal{H}) = m^{\mathcal{H}}(r)$.

Then, using the expectation formula and the concavity of $\log$:
\begin{align*}
    H^S(l_0) &= \E \left[ \frac{1}{n} \sum_{i=0}^{n-1} \log_2 \Delta^S(x_{il_0+1},...,x_{(i+1)l_0}) \right] \\ &\leq  \frac{1}{n} \sum_{i=0}^{n-1} \log_2 \left( \max \Delta^S(x_{il_0+1},...,x_{(i+1)l_0}) \right) \\ &= \frac{1}{n} \sum_{i=0}^{n-1} \log_2 d_{VC}(S) = \log_2 d_{VC}(S).
\end{align*}
\end{proof}

\noindent\textbf{Claim:} For deterministic algorithms, $C_{\mathcal{A}}(g,d)$ becomes the Shannon surprisal of the set of datasets producing $g$, $C_{\mathcal{A}}(g,d) = -\log_2 \D(S)$, where $$S =\{d' \in \mathcal{Z}^n \mid \mathcal{A}(d') = g\}$$ is the collection of datasets $d'$ such that $p(g|d') = 1$ under $\mathcal{A}$.

\begin{proof}
Recall, from Definition 5, that $C_{\mathcal{A}}(g,d) = \log_2\frac{p(g|d)}{p(g)}$. If $p(g|d) = 1$ (i.e., deterministic algs),
\begin{align*}
    \log_2 \frac{p(g|d)}{p(g)} &= -\log_2 p(g)\\
        &= -\log_2 \sum_{d'}p(g,d')\\
        &= -\log_2 \sum_{d'}p(g|d')p(d')\\
        &= -\log_2 \sum_{d':p(g|d') = 1}p(d')\\
        &= -\log_2 \sum_{d':p(g|d') = 1}\D(d')\\
        &= -\log_2 \sum_{d'\in S}\D(d')\\
        &= -\log_2\D(S)
\end{align*}
where $S = \{d' \in \mathcal{Z}^n \mid p(g|d') = 1\}$.
\end{proof}

\end{document}